\documentclass[letterpaper, 10pt, conference]{ieeeconf}  % Comment this line out if you need a5paper

\IEEEoverridecommandlockouts                              % for the \thanks command

\usepackage[basicpaper,lengthhacks]{tinyiu}
\usepackage{xifthen}
\usepackage{stackengine} % for stacking unlabelled caption in subfigure

\def\star{\ensuremath{^{*}}}
\def\RRT*{RRT\star{}}
\def\RRdT*{RRdT\star{}}
\def\PRM*{PRM\star{}}

\makeatletter
\newcommand{\overrightsmallarrow}{\mathpalette{\overarrowsmall@\rightarrowfill@}}
\newcommand{\overarrowsmall@}[3]{%
  \vbox{%
    \ialign{%
      ##\crcr
      #1{\smaller@style{#2}}\crcr
      \noalign{\nointerlineskip}%
      $\m@th\hfil#2#3\hfil$\crcr
    }%
  }%
}
\def\smaller@style#1{%
  \ifx#1\displaystyle\scriptstyle\else
    \ifx#1\textstyle\scriptstyle\else
      \scriptscriptstyle
    \fi
  \fi
}
\makeatother
\newcommand*\Path[1]{\sigma_{\overrightsmallarrow{#1}}}
\newcommand*\closureset[1]{\text{cl}(#1)}
\newcommand*\bLookOut[1]{%
    \ensuremath{\beta}\textsc{\footnotesize{-LookOut}}%
    \ifthenelse{\isempty{#1}}%
    {}{(#1)}%
}

\usepackage[boxed,ruled,vlined,linesnumbered]{algorithm2e}
\DontPrintSemicolon
\SetKwProg{Fn}{function}{}{}
\SetKwFunction{FnSampleFree}{SampleFree}
\SetKwFunction{FnRestartArm}{RestartArm}
\SetKwFunction{FnPickArm}{PickArm}
\SetKwFunction{FnRewire}{Rewire}
\SetKwFunction{FnNearest}{Nearest}
\SetKwFunction{FnRestartArm}{RestartArm}
\SetKwComment{Comment}{$\triangleright$\ }{}
\SetKwInput{KwInit}{Initialise}

\usepackage{cleveref}                                        % Better ref!
\crefname{assumption}{assumption}{assumptions}
\crefname{problem}{problem}{problems}
\crefname{algorithm}{Alg.}{Algs.}
\Crefname{algorithm}{Algorithm}{Algorithms}
\crefname{figure}{Fig.}{Figs.} % ieee figure must be cap
\crefformat{equation}{(#2#1#3)}
\crefrangeformat{equation}{(#3#1#4) to~(#5#2#6)}
\crefmultiformat{equation}{(#2#1#3)}%
{ and~(#2#1#3)}{, (#2#1#3)}{ and~(#2#1#3)}

\usepackage[
    backend=biber
  ,giveninits=true                  % show only initials in ref
  ,url=false, isbn=false, doi=false % Remove unnecessary fields
  ,maxnames=99                       % max authors showing, truncate otherwise (def: 3)
  ,minnames=3                       % after truncated, show this num of authors (def: 1)
  ,sorting=none             % citation number appear as they are used
  ,date=year                % ony show year & ignore month or day field
  ,labeldate=year
  ,style=ieee
  ]{biblatex}
\AtEveryBibitem{                      % ONLY get rid of JSTOR (leave arXiv alone)
  \iffieldequalstr{eprinttype}{jstor}
  {\clearfield{eprint}}
  {}
}
\DeclareSourcemap{
  \maps{
    \map{
      \pertype{inproceedings}
      \step[fieldset=publisher, null]
      \step[fieldset=urldate, null]
      \step[fieldset=volume, null]
      \step[fieldset=pages, null]
    }
  }
}
\renewrobustcmd*{\bibinitdelim}{\,} % this controls the gap between each initials. "\," puts a thin space (def: ~ (full space))
\Crefname{figure}{Fig.}{Figs.}

\title{\LARGE \bf
Balancing Global Exploration and Local-connectivity Exploitation with Rapidly-exploring Random disjointed-Trees
}

\author{Tin Lai$^{*1}$, Fabio Ramos$^{1,2}$ and Gilad Francis$^{1}$% <-this % stops a space
\thanks{
$^{*}${\tt\small tin.lai@sydney.edu.au}
$^{1}$School of Computer Science, The University of Sydney, Australia. 
$^{2}$NVIDIA, USA. 
}%
}

\begin{document}

\maketitle
\thispagestyle{empty}
\pagestyle{empty}

%%%%%%%%%%%%%%%%%%%%%%%%%%%%%%%%%%%%%%%%%%%%%%%%%%%%%%%%%%%%%%%%%%%%%%%%%%%%%%%%
\begin{abstract}

Sampling efficiency in a highly constrained environment has long been a major challenge for sampling-based planners.
In this work, we propose Rapidly-exploring Random disjointed-Trees\star{} (\RRdT*), an incremental optimal multi-query planner.
\RRdT* uses multiple disjointed-trees to exploit local-connectivity of spaces via Markov Chain random sampling, which utilises neighbourhood information derived from previous successful and failed samples.
To balance local exploitation, \RRdT* actively explore unseen global spaces when local-connectivity exploitation is unsuccessful.
The active trade-off between local exploitation and global exploration is formulated as a multi-armed bandit problem.
We argue that the active balancing of global exploration and local exploitation is the key to improving sample efficient in sampling-based motion planners.
We provide rigorous proofs of completeness and optimal convergence for this novel approach.
Furthermore, we demonstrate experimentally the effectiveness of \RRdT*'s locally exploring trees in granting improved visibility for planning.
Consequently, \RRdT* outperforms existing state-of-the-art incremental planners, especially in highly constrained environments.

\end{abstract}
%%%%%%%%%%%%%%%%%%%%%%%%%%%%%%%%%%%%%%%%%%%%%%%%%%%%%%%%%%%%%%%%%%%%%%%%%%%%%%%%

\section{Introduction}

\emph{Sampling-based path planners} (SBPs) provide a robust approach to robotic motion planning, where the objective is to produce a sequence of actions for the system to transits from an initial point to a goal point under a set of constraints.
These planners are favourable as they offer robustness in high-dimensional configuration spaces (\emph{C-space}), since
the sampling procedure replaces explicit construction of the---often intractable---\emph{C-space}.
An SBP samples configuration points randomly and connects valid points in a graph or tree-like structure.
This structure is then searched for a possible solution, and given sufficient time the planner is guaranteed to find a solution if one exists.
This guarantee is often called  \emph{probabilistic completeness} \autocite{elbanhawi2014_SampRobo}.

An SBP can be categorically classified as either single-query or multi-query.
A single-query planner such as Rapidly-exploring Random Tree (RRT) \autocite{lavalle1998_RapiRand} returns a feasible path that connects a pair of initial and goal points.
Multi-query planners such as Probabilistic Roadmap (PRM) \autocite{kavraki1996_ProbRoad}, on the other hand, constructs a topological graph that allows the planner to perform different instances of queries efficiently.
Star variants of these planners, such as RRT\star{} \autocite{karaman2010_IncrSamp} and \PRM* \autocite{karaman2011_SampAlgo}, are considered \emph{asymptotically optimal} \autocite{elbanhawi2014_SampRobo}, as
they further guarantee that the solution will converge, in the limit, to the optimal solution, given a user-defined path associated cost.

Although SBPs are probabilistically complete, their runtime performance is significantly influenced by \emph{C-space}'s complexity.
Intuitively, tightly constrained regions are problematic as they limit the connectivity of free space.
Indeed, this is a widely recognised issue of SBPs, which is often framed as planning in narrow passages \autocite{hsu1998_FindNarr}.
Limited free space within narrow passages implies a low probability of picking a point inside.
Narrow passages severely restrict the performance of single-query algorithms since low-occurring samples that do happen to fall inside narrow passages are merely thrown away if the tree fails to expand.
Meaning, the growth of the tree is restricted by the surrounding obstacles, resulting in limited visibility and high failure rate in extending new connections.
Consequently, the growth of tree will be bottlenecked by narrow passages until a series of successful points within narrowed passages had been sampled.

In this paper, we propose Rapidly-exploring Random disjointed-Trees\star{} (\RRdT*) as a novel anytime SBP \autocite{karaman2011_AnytMoti} based on \RRT*, which inherit \RRT*'s \emph{probabilistic completeness} and \emph{asymptotic optimality} guarantees.
\RRdT* explores \emph{C-space} with multiple locally exploring disjointed trees (d-trees).
The initial roots of the d-trees are scattered randomly in \emph{C-space}, which is followed by expansion of the d-trees, exploiting local connectivity.
 This approach overcomes standard SBPs' failure in creating connections in highly restricted space, as \RRdT* will create a new d-tree at the sampled point if it fails to connect to an existing d-trees.
Consequently, the \RRdT* path planner can be re-framed as a multi-armed bandit (MAB), which balances global exploration and local-connectivity exploitation of \emph{C-space}.

Our contribution is an incremental SBP that exploits local-connectivity to improve performance while balancing exploration.
This is achieved by formulating the balance of the competing choices as an MAB with infinite mortal arms and non-stationary reward sequence.
We provide rigorous proofs on completeness and optimal convergence guarantees.
We show in simulations that such formulation yields superior results with high sampling efficiency in highly constrained \emph{C-space}, tackling the limited visibility issue that many SBPs faces.

\section{Related Work}\label{sec:related-work}

Previous research in SBPs have been focuses on
\begin{enumerate*}[label=\textit{(\roman*)}]
\item runtime for finding an initial solution, and
\item convergence rate of obtaining an optimal solution.
\end{enumerate*}
Improvements to convergence rates are achieved by focusing sampling to certain regions in \emph{C-space}, e.g., restricting search space to a prolate hypersphere
\autocite{gammell2018_InfoSamp},
or adaptively bias toward regions with limited visibility \autocite{yershova2005_DynaRRTs}.
Several SBP algorithms use bridge tests to discover narrow passages \autocite{wilmarth1999_MAPRProb}, followed by dense re-sampling at those regions \autocite{hsu2003_BridTest,sun2005_NarrPass,wang2010_TripRRTs}.
Another approach is retraction-based planners that optimise to generates samples close to the boundary of obstacles \autocite{zhang2008_EffiRetr,lee2012_SRRRSele}.
However, all methods require a user-defined heuristic for finding narrow spaces that naturally has no explicit representation.

Planners using multiple exploring trees were presented previously too, for example, growing bidirectional trees
\autocite{kuffner2000_RRTcEffi},
growing multiple local trees
with a heuristic
tree selection scheme \autocite{strandberg2004_AugmRRTp},
and combining bridge test with a learning technique to model the probability of tree selection \autocite{wang2018_LearMult}.
While the concept of utilising multiple trees had improved sampling efficiency, these algorithms are restrictive in tree locations or require explicit computations for finding narrow passages, which are hard to generalise to high-dimensional \emph{C-space}.

Markov Chain Monte Carlo (MCMC) has also been used by SBPs, as it utilises information observed from previous samples, instead of merely discarding failed attempts \autocite{chen2015_MotiPlan}. \citeauthor*{al-bluwi2012_MotiPlan-1} showed that a PRM is in fact the result of a set of MCMC explorations being run simultaneously \autocite{al-bluwi2012_MotiPlan-1}. A Monte Carlo Random Walk path planner was proposed in \autocite{nakhost2012_ResoPlan}, where a stochastic planner explores the neighbourhood of a search state by constructing a Markov Chain to propose spaces with high contributions.

Several authors proposed SBPs that balance between exploration and exploitation behaviours during planning, borrowing from standard practice in reinforcement learning.
Heuristic biasing was used to guide tree search greedily while retaining probability for random exploration \autocite{urmson2003_ApprHeur}. However, this method does not incorporate valuable information from invalid samples.
\citeauthor*{rickert2008_BalaExpl} \autocite{rickert2008_BalaExpl} proposed a 2-step SBP that exploits \emph{C-space}'s structure after an initial workspace exploration step.
While this approach can significantly improve performance by improving tree visibility, it lacks theoretical \emph{completeness} guarantee.

Our method, \RRdT*, actively balances explorative and exploitative behaviour to achieve high sampling efficiency.
As an incremental multi-query SBP it uses multiple trees to exploit local spaces via chained sampling at unvisited spaces.
Balancing the two objectives leads to enhanced performance while maintaining completeness and convergence guarantees.

\section{Rapidly-exploring Random disjointed-Trees\star{}}\label{sec:main}

In this section, we formalise the path planning problem and describe the \RRdT* algorithm.

\begin{figure}[!tb]
    \centering
    \begin{subfigure}{0.8\linewidth}
        \includegraphics[width=.99\linewidth]{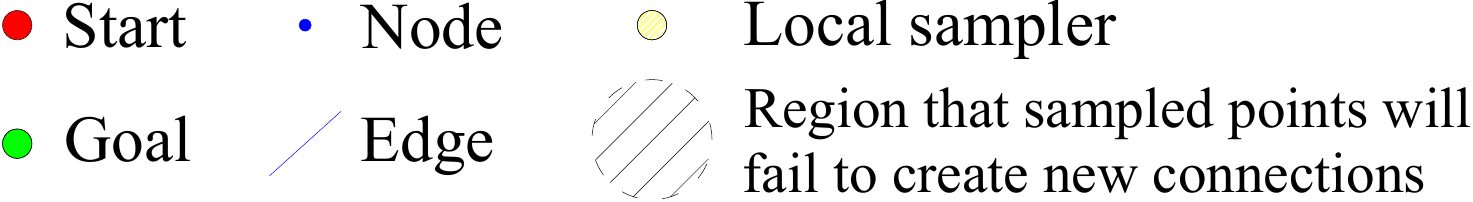}
    \end{subfigure}
    \par\vspace{0.5mm}
    \begin{subfigure}{.45\linewidth}
        \includegraphics[width=.99\linewidth]{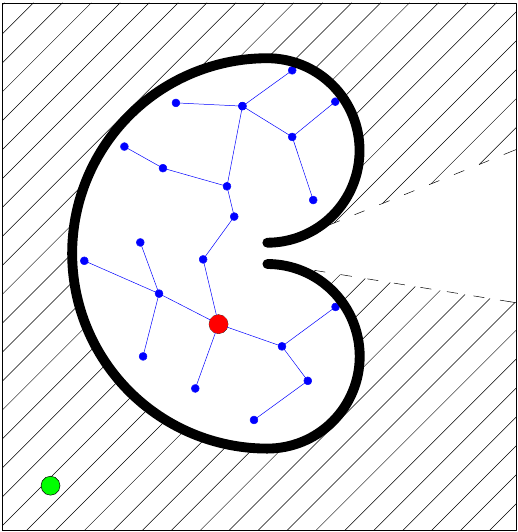}
        \caption{RRT \label{fig:bug-trap-problem:rrt}}
    \end{subfigure}%
    \begin{subfigure}{.45\linewidth}
        \includegraphics[width=.99\linewidth]{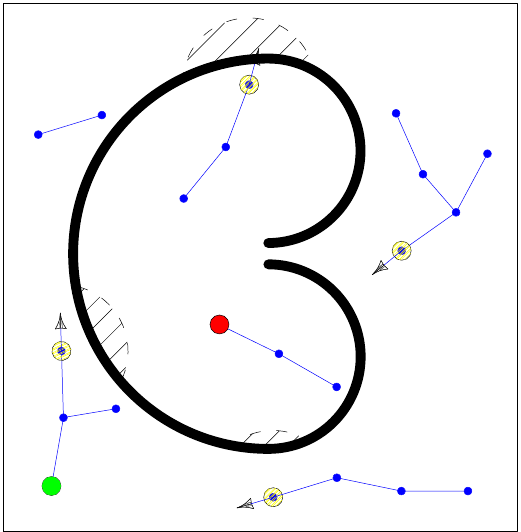}
        \caption{RRdT \label{fig:bug-trap-problem:rrdt}}
    \end{subfigure}
    \caption{
        Sampling scenarios in the bug-trap problem \autocite{yershova2005_DynaRRTs} with same number of nodes.
        (\subref{fig:bug-trap-problem:rrt}) RRT:
        Expansions of tree being restricted by surrounding obstacles, has very limited visibility (white regions) to utilise sampled points.
        (\subref{fig:bug-trap-problem:rrdt}) RRdT:
        Large visible region with high probability of creating new connections.
        Local samplers can exploit local-connectivity in parallel with less failures.
        \label{fig:bug-trap-problem}
    }
\end{figure}

\subsection{Problem formulation}

Let $C\subseteq\mathbb{R}^d$ denotes the \emph{C-space}, where $d\ge 2$; and
$C_{obs} \subseteq C$ denotes the set of invalid states.
The obstacle-free space $C_{free}$ is the closure set $\closureset{C \setminus C_{obs}}$.
Let $q\in C$ denotes a state in \emph{C-space}.
A feasible path is formally defined as follows:

\begin{definition}[Feasible path] \label{def:feasible-path}
A path is a sequence of consecutively connected configurations $\sigma:[0,1]\to C_{free}$.
A path is said to be feasible if $\sigma(0)=q_{init}$, $\sigma(1)=q_{goal}$, and it is collision free: $\sigma(\tau)\in C_{free} \forall \tau\in[0,1]$.
\end{definition}

The objective of a path planning algorithm is to construct a feasible path $\sigma$ from an initial point $q_{init}$ to a goal point $q_{goal}$. Therefore, path planning can be considered as one of the following problems:
\begin{problem}[Feasibility planning] \label{problem:feasibility}
Given a configuration space $C$,
a free space $C_{free}$,
an obstacle space $C_{obs}$,
an initial configuration $q_{init} \in C_{free}$,
and a goal configuration $q_{goal} \in C_{free}$.
Find a path that satisfy \cref{def:feasible-path}.
\end{problem}

\begin{problem}[Optimal planning] \label{problem:optimality}
Let there be a \emph{cost function} $c(\sigma)$ that assigns non-negative cost to all non-trivial feasible paths defined by \cref{def:feasible-path}.
For all possible paths from \cref{problem:feasibility}, find $\sigma^*:[0,1] \to \closureset{C_{free}}$, such that $\sigma^*(0) = q_{init}$, $\sigma^*(t) = q_{goal}$, and $c(\sigma^*) = \min_{ \sigma \in \closureset{C_{free}}} c(\sigma)$.
\end{problem}

\subsection{High-level description}

\RRdT* aims to strike a balance between global exploration and local exploitation with effective utilisation of information obtained from sampled points.
RRT%
    \footnote{Although \RRdT* is based on \RRT*, the concept is interchangeable with RRT, which will be used for ease of illustration.} tree expansion is limited to a local scope bounded by the neighbourhood visibility of the tree nodes as illustrated in \cref{fig:bug-trap-problem:rrt}.
It continuously tries to grow the root tree outwards by sampling random points and creating new connections to the closest nodes.
Therefore, RRT will often reject updates from valid samples when there exists no free route from the closest existing node towards a sampled point.
For example, in \cref{fig:bug-trap-problem:rrt}, any sampled point that falls inside the hatched region will be discarded due to failure in tree expansion.
This problem is exacerbated in a tightly constrained \emph{C-space}, where there are not many readily available free routes.
\RRdT* maintains high visibility by using d-trees to explore \emph{C-space}.
It exploits local connectivity by employing local samplers (yellow circle in \cref{fig:bug-trap-problem:rrdt}) that serve as density estimators.
D-trees explore within $C_{free}$ by estimating which portions of its surrounding have the highest probability of being free space.
Thus, maintaining local connectivity information without wasting samples.
As illustrated in \cref{fig:bug-trap-problem:rrdt}, \RRdT* mitigates the limited visibility by creating new local samplers in regions, which are not visible from the root tree.
Local sampling in \RRdT* is performed by moving an $\epsilon$-distance step drawn from the local sampler.
By employing local sampling, \RRdT* will only fail to create new connections at non-visible locations within the $\epsilon$-ball, which has a smaller volume compared to the entire \emph{C-space}.
Sampling locally take advantages of the local-connectivity in narrow passages while avoiding the needs to use heuristic measurement to identify narrow passages \autocite{wang2018_LearMult}.

\RRdT* can be used for re-planning, and it returns a feasible solution faster than general multi-query planners such as PRM.
PRM performs shortest path search only after exhausting its sampling budget.
In contrast, \RRdT* is incremental and maintains an exploring tree from a root that is continuously adding new samples.
The balance between exploration and exploitation fuses the advantageous properties from the two main SBPs---tree connectivity from RRT and high visibility from PRM.
Effective balancing can ensure faster convergence rate while maintaining asymptotic \emph{completeness}.
The proposal of \RRdT* aims to fill in this gap with provable guarantees.

\subsection{Balancing exploration and exploitation}

Exploiting local connectivity comes naturally in path planning, where the objective is to construct a directed path that connects two given points.
Let $\sigma_{sol}$ denotes a potential solution path,
$\Path{q_i q_j}$ denotes a path with $\sigma(0)=q_i,\sigma(1)=q_j$,
and $\mathcal{V}(q)$ denotes the visibility set of $q$, which represents the region of $C_{free}$ visible from some $q \in C_{free}$.
Intuitively, each point $q_i \in \sigma_{sol}$ is connected to at least two other points (parent and child node) that lie on the same solution.
That is, each point on the solution path has at least two exploitable locally connected nodes;
only exceptions are $q_{init}$ and $q_{goal}$, which does not has a parent and child node respectively.
\begin{equation}
\begin{aligned}
    \exists^{\ge 2}q_j, q_j\in \mathcal{V}(q_i) \land q_i \ne q_j \land \Path{q_i q_j} \in C_{free} \\
    \forall q_i \in \sigma_{sol}\setminus \{q_{init},q_{goal}\}
\end{aligned}
\end{equation}

This statement of at least two local connections is a lower bound that remains true for all path planning settings.
Higher visibility implies more available local connections to exploit.
However, most SBPs do not utilise---or exploit---this fact to avoid the pitfall of being stuck in dead-ends;
for example,
random sampling
by PRM ignores valuable local connectivity information.
On the other hand, heuristic-based algorithms such as potential field method \autocite{khatib1990_RealObst} tries to exploit obstacle information, but often will not find a feasible path even if one exists \autocite{koren1991_PoteFiel}.
The proposed \RRdT* strikes a balance between exploiting the locality of \emph{C-space} and avoid being trapped within obstacles by actively balancing the exploration of unvisited space, as we describe below:

\begin{figure}[tb]
    \centering
    \begin{tikzpicture}
        \node (img) {\includegraphics[width=.8\linewidth]{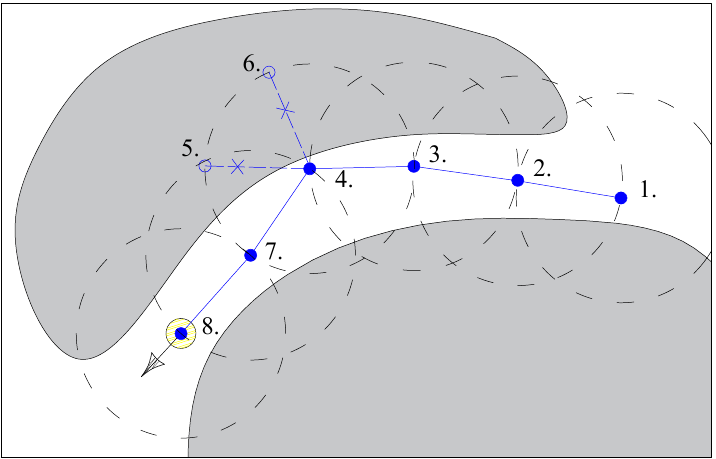}};
        \node at (-2.5,0.6){$C_{obs}$};
        \node at (1.75,-1.5){$C_{obs}$};
    \end{tikzpicture}%
    \caption{
        Local connectivity exploitation using MCMC random walk inside narrow passage.
        With a local sampler first initialised at location 1, 2-4 are successive and successful samples.
        The 5\textsuperscript{th} sample fails the collision check.
        Local sampler then retries to re-sample sampler's proposal distribution (6\textsuperscript{th} sample) and fails again.
        The 7\textsuperscript{th} sample finds a free route to continue the chained sampling.
        If local sampler get stuck during the process (e.g. at a dead-end), MAB scheduler will restart arm at a new location.
        \label{fig:rrdt-local-sampler-example}
    }
\end{figure}

\subsubsection{Exploitation on local-connectivity}

Local exploitation in \RRdT* is performed by a \emph{local sampler}---an adaptive density estimator that defines a proposal distribution based on previous successes, which is utilised by an MCMC random walker.
This attempts to exploit the local connectivity between each point as a chained sampling procedure built from the previous movement's direction.
\Cref{fig:rrdt-local-sampler-example} demonstrates a local sampler navigating within a narrow passage, utilising previous successful
observations to adapt its proposal distribution.

\subsubsection{Exploration on global-space}

Exploration is implicitly performed by restarting a local sampler at a new location.
Global exploration in \RRdT* utilise generic uniformly random sampling of \emph{C-space}, which inherits \emph{probabilistic completeness} and asymptotic guarantees of optimal convergence from \RRT*.
\RRdT* as a framework does not use a global dynamic sampling schedule.
Replacing uniform schedule with a dynamic schedule, such as the one used in Informed \RRT* \autocite{gammell2018_InfoSamp}, is trivial, but is beyond the scope of this paper.

\subsubsection{Balancing scheduler}\label{sec:mab-scheduler}

Balancing global exploration and local exploitation in \RRdT* can be formulated as an MAB.
The objective of an MAB scheduler is to allocate resources effectively based on the complexity of \emph{C-space}.
Spaces with a complex environment can utilise local-connectivity information to exploits obstacles' structure;
whilst spaces with high visibility can allocate more resources on rapid exploring.

We employ Mortal MAB \autocite{chakrabarti2009_MortMult} with non-stationary reward sequences \autocite{besbes2014_StocMult} and infinitely many arms \autocite{berry1997_BandProb} as \RRdT*'s MAB scheduler.
An arm $a_{i}$ represents a local sampler that moves within $C_{free}$;
its state at time $t$ is denoted as $a_{i,t}$ where it will samples $C$ to makes an observation $o_t$, and receives a reward $R(a_{i,t})$.
We consider discrete time and Bernoulli arms, with success probability as
\begin{equation}
    \mathbb{P}(a_{i,t} \mid a_{i,t-1}, o_{t-1})\,\forall_{i\in\{1,...,k\}}. \label{eq:arm-probability}
\end{equation}
We consider the arms as independent random variables with $\mathbb{P}(a_{i,t})$ having non-stationary distribution (as each arm represents a moving density estimator of connections in $C_{free}$), which depends on its previous successes.
The reward sequence of $a_i$ is a stochastic sequence with unknown payoff distributions and changes rapidly according to the complexity of \emph{C-space}.
The reward $R(a_{i,t})$ is a mapping of \emph{C-space}, where $a_{i,t}$ is mapped to a pay-off of 1 if $a_{i} \in C_{free}$ at $t$, and 0 otherwise.
An arm is `mortal' in a sense that it is used to model a local sampler that moves around in \emph{C-space};
thus it has a stochastic lifetime after which they expire due to events such as running into a dead-end or stuck in a narrow passage.
The action space is infinite ($k=\infty$) because a new arm can be created on demand by randomly sampling a new point, especially after an arm is expired.

\subsection{Implementation}

\begin{algorithm}[tb]
    \caption{MAB arm restart scheduler} \label{alg:mab-restart}
    \Fn{$\FuncSty{RestartArm}(A,D,\eta)$}{
        % \ForEach{$a_i$ in $\set{a_j \in A \mid \mathbb{P}(a_j) < \eta}$}{
        \ForEach{arm $a_i\in A$ with $\mathbb{P}(a_i) < \eta$}{
            $q_{rand} \gets \FnSampleFree$\;
            Join $q_{rand}$ to all $d_i \in D$ within $\epsilon$-distance\; \label{alg:mab-restart:epsilon-joins-1}
            \If{Not joined to existing tree}{ \label{alg:mab-restart:epsilon-joins-2}
                Create new arm $a_{new}$ at $q_{rand}$ \; \label{alg:mab-restart:epsilon-joins-3}
                % $A \gets A \cup \{a_{new}\} \setminus a_i$ \; \label{alg:mab-restart:epsilon-joins-4}
                Replace $a_i$ in $A$ with $a_{new}$ \; \label{alg:mab-restart:epsilon-joins-4}
            }
            \Return{True} \Comment*[r]{new node is added}
        }
        \Return{False} \Comment*[r]{No sampler restarted}
    }
\end{algorithm}

\Cref{alg:mab-restart} shows \RRdT*'s MAB restart scheduler.
The scheduler checks whether the probability of each arm $a_i$ in the set of arms $A$ is lower than $\eta$, an arbitrarily small positive constant. A low probability indicates successive failures (e.g. stuck in a dead-end), which results in the arm randomly relocated to a new location in $C_{free}$. The new location, $q_{rand}$,  is assigned by global exploration sampling schedule, \FuncSty{SampleFree}.
However, if the new sampled point, $q_{rand}$, is within $\epsilon$-distance of an existing d-tree $d_i$ from the set of trees $D$, the new point will be added to the appropriate d-tree instead
    (\cref{alg:mab-restart:epsilon-joins-1,alg:mab-restart:epsilon-joins-2,alg:mab-restart:epsilon-joins-3,alg:mab-restart:epsilon-joins-4}).
The resulting behaviour is identical to \RRT* (extending existing node $\epsilon$-distance towards sampled point), which implies \RRdT*'s asymptotic behaviour converges to \RRT* after \emph{C-space} is covered by d-trees.

\Cref{alg:rrdt} defines the main \RRdT* algorithm. \RRdT* exploration starts by placing $k$ arms randomly in $C_{free}$.
Then, it draws an arm from a multinomial distribution in \FuncSty{PickArm}, with each arm $a_i \in A$ having probability formulated by \cref{eq:arm-probability}.
Once an arm had been picked, \RRdT* exploits the arm's local-connectivity by sampling from the local proposal distribution (\cref{alg:rrdt:local-sampling}).
To incorporate past sampling successes, we employed an MCMC random walker with a von Mises-Fisher (vMF) \autocite{fisher1993_StatAnal} distribution, which is a wrapped Gaussian distribution. %, commonly used for directional statistics.
It enables us to construct a chained directed sampling using a random walker that produces a sample $q_{new}$.
If $q_{new}$ is valid, the arm's position $a_i.pos$ is updated accordingly  (\cref{alg:rrdt:update-arm-pos}).
A schematic example of a  directed local-connectivity search of $C_{free}$ is shown in \cref{fig:rrdt-local-sampler-example}.
An illustration of local sampling at different locations simultaneously is shown in \cref{fig:map:room}.

Asymptotic optimality of \RRdT* is maintained by \FuncSty{Rewire} (\cref{alg:rrdt:rewire}), a procedure that rewires existing edges to neighbourhood edges with the least accumulated cost \autocite{karaman2010_IncrSamp}.
Global exploration in \RRdT* is implicitly performed when the MAB scheduler restarts an arm at a randomly chosen location in $C_{free}$ (\cref{alg:rrdt:restart-arm}).
Thus, exploration is performed when local exploitation is unsuccessful, which allows \RRdT* to actively adjust its immediate objective based on \emph{C-space}'s complexity.

\begin{algorithm}[tb]
    \caption{\RRdT* Algorithm} \label{alg:rrdt}
\Indmm\Indmm
    \KwIn{$q_{init},q_{goal},N,k,\epsilon,\eta$}
     \KwInit{$Root \gets G(V=\set{q_{init}}, E=\emptyset)$;
     $D \gets \set{Root}$; $A \gets \emptyset$; $n \gets 1$\;}
\Indpp\Indpp
    Initialise $k$ arms into $A$ and d-trees into $D$\;
    \While{$n \le N$}{

        \uIf(\Comment*[f]{\cref{alg:mab-restart}}){\FnRestartArm{$A,D,\eta$}}{ \label{alg:rrdt:restart-arm}
            $n \gets n + 1$\;
        }
        \Else{
            $a_{i} \gets \FnPickArm{A}$\; \label{alg:rrdt:pickarm}
            $q_{new} \gets a_{i}$ samples locally via MCMC\; \label{alg:rrdt:local-sampling}
            \If{$\Path{a_{i}.pos\;q_{new}} \in C_{free}$}{
                Join $q_{new}$ to all $d_i \in D$ within $\epsilon$-distance\; \label{alg:rrdt:join-nearby-tree}
                \lIf{$q_{new}$ joined to $Root$}{\FnRewire{Root}} \label{alg:rrdt:rewire}
                Update $a_{i}.pos$ to $q_{new}$\; \label{alg:rrdt:update-arm-pos}
                $n \gets n + 1$\;
            }
            Updates $a_{i}$ probability\;
        }
    }
\end{algorithm}

\section{Analysis}\label{sec:analysis}

We will first define the notion of \emph{expansiveness} \autocite{hsu1997_PathPlan}.

\begin{definition}[expansiveness] \label{def:expansive}
    Let $\mu(X)$ denotes the Lebesgue measure of a set $X$, which represents its volume.
    The free space $C_{free}$ is said to be $(\alpha,\beta,\epsilon)$-\emph{expansive} \autocite{hsu1997_PathPlan} if each of its connected components $C_{free}' \subseteq C_{free}$ satisfy the following conditions:
    \begin{enumerate}[label=(\roman*)]
        \item for every point $q \in C_{free}'$, $\mu(\mathcal{V}(q)) \ge \epsilon$
        \item for any connected subset $S \subseteq C_{free}'$,
        the set \\
        {\small
            $\bLookOut{S}=\{q \in S\;|\;\mu(\mathcal{V}(q) \setminus S)\ge \beta \mu(C_{free}'\setminus S)\}$
        }
        has volume
        $\mu(\bLookOut{S}) \ge \alpha \mu(S)$
    \end{enumerate}
\end{definition}

This notation guarantees $C_{free}$ is $\epsilon$-good \autocite{lavalle2001_RandKino};
and measures complexity of $C$ with the quantity of visibility by $q\in S$.

\subsection{Feasibility planning}

Let $n_{dtree}$ denotes the total number of d-trees restart in the time-span of the algorithm.
Then, although a new d-tree is restarted whenever a new arm is added, $n_{dtree}$ will be bounded by a real value constant.
Consider the followings:

\begin{assumption}\label{assum:c-free-finite-set}
The free space $C_{free} \subseteq C$ is a finite set.
\end{assumption}

\begin{lemma}[Termination of d-tree restarting] \label{lemma:termination-dtree-spawning}
Let \Cref{assum:c-free-finite-set} hold.
Then, the total number of d-trees restarted for any given \emph{C-space} is always finite.
That is, there exists a constant $\phi\in\mathbb{R}$ such that $n_{dtree} < \phi$ for any given $C$.
\end{lemma}
\begin{proof}
A new d-tree will be restarted by sampling a new point in $\mathbb{R}^d$.
From \cref{alg:mab-restart}, if there exists a d-tree within $\epsilon$-distance, the new point will be added as a node to that d-tree instead of creating a new d-tree.
Hence, every new d-tree must be at least $\epsilon$-distance away from each other.
We can formulate each d-tree as an $\epsilon$-ball centred at the d-tree's origin.
In the limiting case, $C_{free}$ will be filled by the volume of $\epsilon$-balls.
Let $\epsilon(q)$ denotes the $\epsilon$-ball of $q$, and $V$ be the set of all nodes. Then, the above statement is formally defined as:
\begin{equation}
    \lim_{t\to\infty} C_{free} \setminus \bigcup_{q \in V_t} \epsilon(q) = \emptyset
\end{equation}
Therefore, $n_{dtree}$ is upper bounded by how many $\epsilon$-balls can $C_{free}$ fits. % before the $\epsilon$-balls fully filled $C_{free}$'s volume.
It is immediate that if $C_{free}$ is a finite set, $n_{dtree}$ will always be bounded by a constant.
This tractability guarantee ensures that using d-trees to explore $C_{free}$ will always terminate.
\end{proof}

\begin{assumption}\label{assum:non-zero-discount}
The MAB scheduler has a discounting factor such that the probability of all arms will eventually decay to zero.
\end{assumption}
\begin{lemma}[Infinite random sampling] \label{lemma:dtree-inf-uniform-sampling}
Let \Cref{assum:non-zero-discount} holds, and $n_{i,t}$ denotes the number of uniformly random points added to a d-tree $d_i \in D$ at time $t$.
Then, $n_{i,t}$ always increases without bound, i.e., as $t\to\infty$, $n_{i,t} \to \infty$.
\end{lemma}
\begin{proof}
\Cref{assum:non-zero-discount} ensures that the probably of all arms in $A$ will eventually decay to zero, no matter how successful they are.
Therefore, the MAB restart scheduler, as defined in \cref{alg:mab-restart}, will perform infinitely many times.
Therefore, there exist infinitely-many uniformly random samples, and $n_{i,t}$ always increases without bound.
\end{proof}

Let $n^{\RRdT*}_t$ and $n^{\RRT*}_t$ be the number of uniformly random points sampled at time $t$ for \RRdT* and \RRT* respectively.
The differences of the asymptotic sampling scheme between \RRdT* and \RRT* is at most a constant.
\begin{theorem}\label{thm:asym-same-as-rrt}
There exists a constant $\phi \in \mathbb{R}$ such that
  \begin{equation}
    \lim_{t\to\infty} \mathbb{E}\left[ \frac{n^{\RRdT*}_t}{n^{\RRT*}_t} \right] \le \phi.
  \end{equation}

\end{theorem}
\begin{proof}
There exists two different sampling schemes being employed in \RRdT*.
\begin{enumerate}[label=(\roman*)]
\item \textbf{Global}: \emph{Random sampling} when an arm restarts (\cref{alg:mab-restart})\label{sampling:random}%
\item \textbf{Local}: \emph{MCMC random walk} sampling when an arm exploits local neighbourhood (\cref{alg:rrdt} \cref{alg:rrdt:local-sampling}).\label{sampling:MCMC}
\end{enumerate}
The total number of uniformly random points sampled for \RRdT* is the summation of each individual d-tree's randomly sampled points,
i.e., $n^{\RRdT*}_t = \sum_{i=1}^{|D|} n_{i,t}$.
Deriving from \Cref{lemma:dtree-inf-uniform-sampling}, it guarantees each d-tree $d_i \in D$ will always has infinite uniformly sampled points $n_{i,t}$ to improve the tree structure.
Hence, $\lim_{t\to\infty} n^{\RRdT*}_t = \infty$.
Therefore, no matter how successful \ref{sampling:MCMC} is, it will be dominated by \ref{sampling:random} in the limiting case.
Therefore, the ratio of the total number of uniformly random points between \RRdT* and \RRT* will be bounded by a constant as the behaviour of \RRdT* will converge to \RRT* as $t\to\infty$.
\end{proof}

With \Cref{lemma:dtree-inf-uniform-sampling,thm:asym-same-as-rrt}, the \emph{probabilistic completeness} of \RRdT* is immediate.
\begin{theorem}[Probabilistic Completeness] \label{thm:prob-completeness}
\RRdT* inherits the same probabilistic completeness of \RRT*.
That is, if there exists a feasible solution to \cref{problem:feasibility}, then $\lim_{i\to\infty}\mathbb{P}(\{\sigma_{sol}(i)\cap q_{goal}\ne\emptyset\})=1$
\end{theorem}

\subsection{Asymptotic optimality planning}

\begin{theorem}[see \autocite{wang2018_LearMult}] \label{thm:ktree-join-bound}
Select $k$ points $q_1,...,q_k$ randomly from $C_{free}'$ including $q_{init}$ and $q_{goal}$. Set $q_1,...,q_k$ as root nodes and extend $k$ trees from these points.
Let $n$ be the total number of nodes that all these trees extended, and $\gamma\in\mathbb{R}$ be a real number in $(0,1]$.
If $n$ satisfies:
\begin{equation}
\begin{split}
    n \ge\, & k(\alpha\beta\epsilon)^{-1} \ln[4(1-\epsilon)]\ln\{3\ln[2k^2(1-\epsilon)]/\gamma\beta\} \\
    & + k\mu(C_{free})/\mu(C)\ln(3k^2/2\gamma)
\end{split}
\end{equation}
then the probability that each pair of these $k$ trees can be attached successfully is at least $1-\gamma$.
\end{theorem}

Therefore, joining of d-trees in \RRdT* confirms to the same bound.
Due to space constrains, please refers to \autocite{wang2018_LearMult} for full proof of \Cref{thm:ktree-join-bound}.

\begin{theorem}[Joining of d-trees]\label{thm:jointrees}
Let $C_{free}' \subseteq C_{free}$ be a connected free space; $d_i,d_j \in D$ be an instance of d-tree in $C_{free}'$ where $d_i \ne d_j$.
If there exists a feasible path to connect $d_i$ to $d_j$, both trees will eventually join as $t \to \infty$.
\end{theorem}
\begin{proof}
\Cref{lemma:termination-dtree-spawning} states that the restarting scheme of d-trees will eventually terminate; hence, the number of d-trees is finite.
\Cref{lemma:dtree-inf-uniform-sampling} ensures the number of uniformly random points sampled for all $d_i \in D$ is infinite as $t\to\infty$.
Hence, in the limiting case, all d-trees satisfy the real-valued bound given by \Cref{thm:ktree-join-bound}.
Therefore, all d-trees in $C_{free}'$ will eventually join to a single tree.
\end{proof}

\begin{theorem}[Asymptotic optimality]
Let $\sigma^{\RRdT*}_t$ be the solution returned by \RRdT* at time $t$, and $c^*$ is the minimal path cost for \cref{problem:optimality}.
If a solution exists, then the cost of $\sigma^{\RRdT*}_t$ will converges to optimal cost almost-surely. That is:
  \begin{equation}
    \mathbb{P} \left( \left\{ \lim_{t\to\infty} c(\sigma^{\RRdT*}_t) = c^* \right\} \right) = 1
  \end{equation}
\end{theorem}
\begin{proof}
From \Cref{thm:jointrees}, all d-trees in the same $C_{free}'$ will converge to a single tree.
According to \Cref{lemma:dtree-inf-uniform-sampling} there will be infinite sampling available to improve that tree;
and together with adequate rewiring procedure \autocite{karaman2010_IncrSamp} by \cref{alg:rrdt} \cref{alg:rrdt:rewire},
it is guaranteed that the solution will converge to the optimal solution as $t\to\infty$.
\end{proof}

\section{Experimental Results}\label{sec:experimental-results}

\begin{figure}
    \centering
    \begin{subfigure}{0.33\linewidth}
        \includegraphics[width=.99\linewidth]{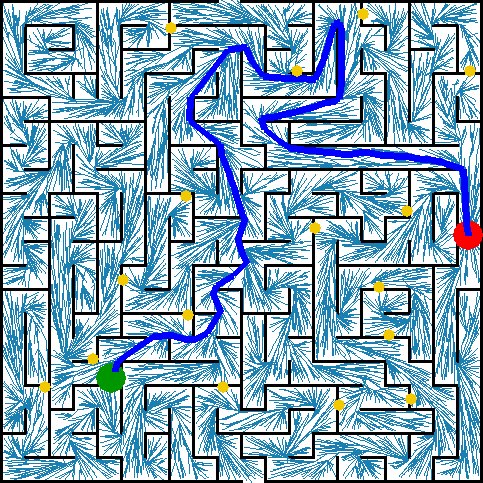}
        \caption{\RRdT* \label{fig:map:maze:rrdt}}
    \end{subfigure}%
    \begin{subfigure}{0.33\linewidth}
        \includegraphics[width=.99\linewidth]{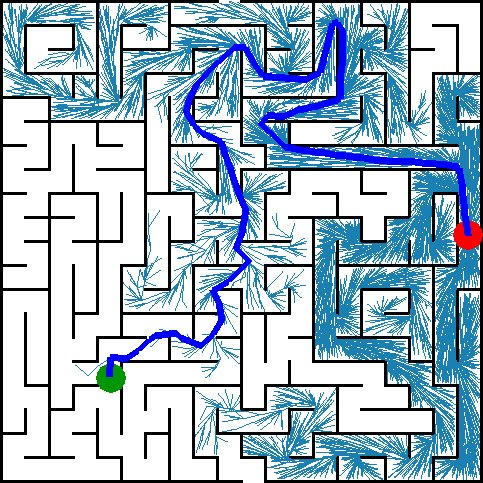}
        \caption{\RRT* \label{fig:map:maze:rrt}}
    \end{subfigure}%
    \begin{subfigure}{0.33\linewidth}
        \includegraphics[width=.99\linewidth]{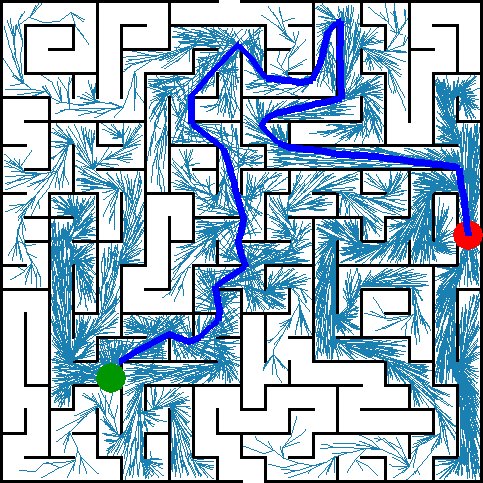}
        \caption{Bi-\RRT* \label{fig:map:maze:birrt}}
    \end{subfigure}
    \caption{Maze environment.
    Incremental SBPs with 10,000 nodes.
    (\subref{fig:map:maze:rrdt}) edges evenly distributed, and visible to all $C_{free}$.
    (\subref{fig:map:maze:rrt}) uneven edges distribution with limit visibility.
    (\subref{fig:map:maze:birrt}) same as \RRT* but edges can grow from two points.
    \label{fig:map:maze}
    }
\end{figure}

\begin{figure}[tb]
    \centering
    \begin{subfigure}{0.65\linewidth}
        \frame{
        \includegraphics[trim=35mm 40mm 35mm 30mm,clip=true,height=4cm]{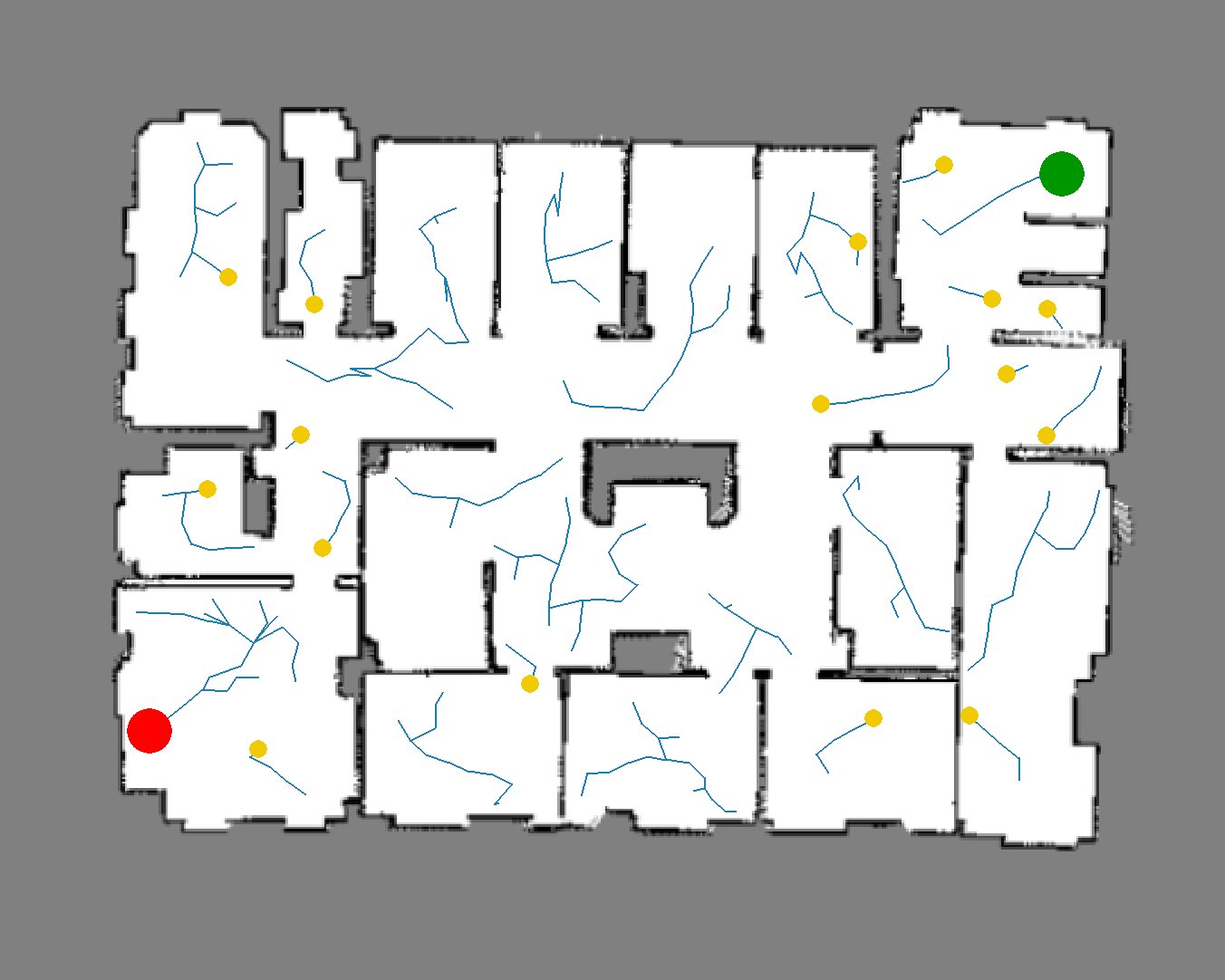}}
        \caption{Standard room floor-plan \label{fig:map:room}}
    \end{subfigure}%
    \hspace{0.5em}%
    \begin{subfigure}{0.25\linewidth}
        \frame{\includegraphics[trim={0 0.8em 0 5em },clip=true,height=4cm]{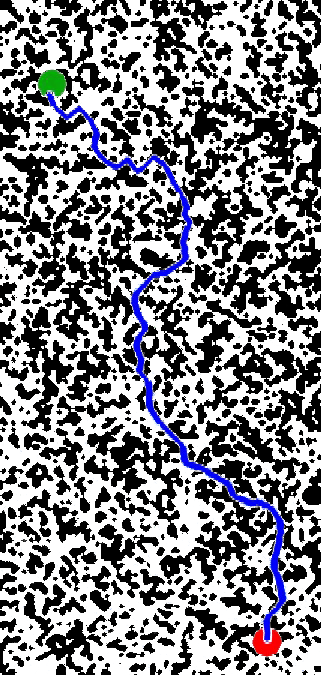}}
        \caption{Clutter \label{fig:map:clutter}}
    \end{subfigure}
    \caption{Environments.
    (\subref{fig:map:room}) Room:
    typical map for baseline (least constrained); also showcasing \RRdT* local samplers exploit multiple local-spaces simultaneously.
    (\subref{fig:map:clutter}) Clutter: randomly generated, extremely limited visibility.
    \label{fig:map}}
\end{figure}

\begin{figure}[tb]
    \centering
    \begin{subfigure}{.99\linewidth}
        \includegraphics[width=.99\linewidth]{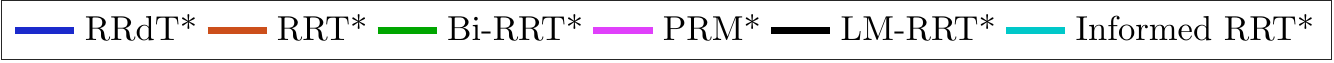}
    \end{subfigure}
    \begin{tikzpicture}
    \node (img) {
        \begin{subfigure}{.31\linewidth}
            \includegraphics[width=.99\linewidth]{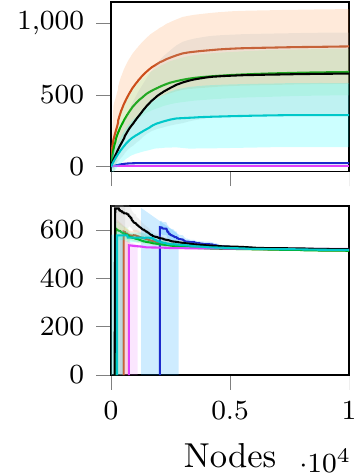}
            \caption{Room \label{fig:alg-comparasion:room}}
        \end{subfigure}%
        \begin{subfigure}{.31\linewidth}
            \includegraphics[width=.99\linewidth]{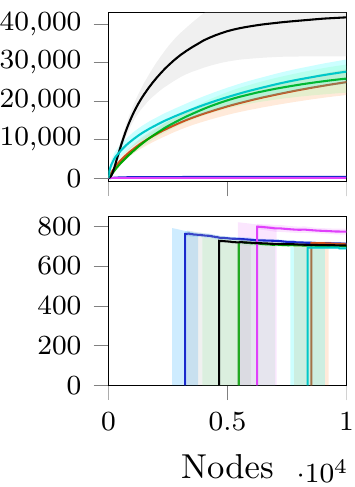}
            \caption{Maze \label{fig:alg-comparasion:maze}}
        \end{subfigure}%
        \begin{subfigure}{.31\linewidth}
            \includegraphics[width=.99\linewidth]{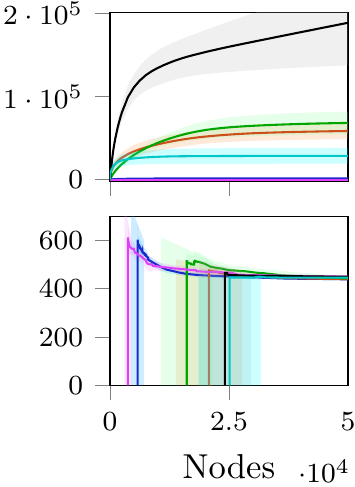}
            \caption{Clutter \label{fig:alg-comparasion:clutter}}
        \end{subfigure}
    };
    \node [rotate=90] at (-4.1,1.3){\footnotesize Failed Conn.};
    \node [rotate=90] at (-4.1,-0.29){\footnotesize Path Cost};
    \end{tikzpicture}%
    \caption{
        Comparison of failed connections (top) and path cost (bottom) as a function of tree nodes in various planning scenarios, shaded region indicates the standard deviation of the 20 repetitions.
        \RRdT* maintains a low failed connections rate in all environments, whereas \PRM* does not creates any connections while sampling.
        (\subref{fig:alg-comparasion:room}) Room:
            simple environment, \RRdT* requires more nodes to return initial solution, however, converges to the same solution.
        (\subref{fig:alg-comparasion:maze}) Maze:
            more complex, \RRdT* reached initial solution faster than other planners.
        (\subref{fig:alg-comparasion:clutter}) Clutter:
            highly complex environment, \RRdT* consistently requires similar amount of nodes and outperforms other incremental SBPs.
        \label{fig:alg-comparasion}
    }
\end{figure}

\begin{table}[tb]
\centering
\caption{Sampled points for 10,000 nodes. Failure comparison. ($\mu\pm2\sigma$) with two significant figures \label{table:sampled-pts-results}}
\begin{tabular}{ccccc}
\toprule
\multirow{2}{*}{Env.}    & \multirow{2}{*}{Method}          & \multicolumn{3}{c}{Sampled Points (\e{3})}                                                                                                \\ \cmidrule{3-5}
                         &                                  & Total         & \begin{tabular}[c]{@{}c@{}}Fail to\\ connect\end{tabular} & \begin{tabular}[c]{@{}c@{}}Points in\\ $C_{obs}$\end{tabular} \\ \toprule
\multirow{6}{*}{Room}    & \multicolumn{1}{c|}{\RRdT*}      & 21$\pm$0.32                                       & \textbf{0.023$\pm$0.010}                    & 11$\pm$0.32                             \\
                         & \multicolumn{1}{c|}{\RRT*}       & 22$\pm$0.88                                       & 0.84$\pm$0.52                               & 11$\pm$0.70                             \\
                         & \multicolumn{1}{c|}{Bi-\RRT*}    & 21$\pm$0.64                                       & 0.66$\pm$0.32                               & 11$\pm$0.54                             \\
                         & \multicolumn{1}{c|}{LM-\RRT*}    & 21$\pm$0.83                                       & 0.69$\pm$0.29                               & 11$\pm$0.59                             \\
                         & \multicolumn{1}{c|}{Informed-\RRT*} & \textbf{18$\pm$0.79}                           & 0.36$\pm$0.22                               & \textbf{7.5$\pm$0.57}                             \\
                         & \multicolumn{1}{c|}{\PRM*}       & 21$\pm$0.30                                       & N/A                                         & 11$\pm$0.30                             \\ \midrule
\multirow{6}{*}{Maze}    & \multicolumn{1}{c|}{\RRdT*}      & \textbf{12$\pm$0.11}                              & \textbf{0.29$\pm$0.052}                     & 1.7$\pm$0.10                            \\
                         & \multicolumn{1}{c|}{\RRT*}       & 40$\pm$6.8                                        & 25$\pm$6.6                                  & 5.3$\pm$1.0                             \\
                         & \multicolumn{1}{c|}{Bi-\RRT*}    & 41$\pm$6.8                                        & 26$\pm$6.8                                  & 5.3$\pm$1.0                             \\
                         & \multicolumn{1}{c|}{LM-\RRT*}    & 50$\pm$8.6                                       & 35$\pm$7.5                               & 5.3$\pm$1.1                             \\
                         & \multicolumn{1}{c|}{Informed-\RRT*} & 36$\pm$4.0                           & 22$\pm$1.7                               & 4.3$\pm$2.3                             \\
                         & \multicolumn{1}{c|}{\PRM*}       & \textbf{12$\pm$0.090}                             & N/A                                         & \textbf{1.5$\pm$0.088}                  \\ \midrule
\multirow{6}{*}{Clutter} & \multicolumn{1}{c|}{\RRdT*}      & 22$\pm$0.29                                       & \textbf{1.3$\pm$0.11}                       & 11$\pm$0.26                          \\
                         & \multicolumn{1}{c|}{\RRT*}       & 100$\pm$22                                         & 42$\pm$16                                   & 48$\pm$15                               \\
                         & \multicolumn{1}{c|}{Bi-\RRT*}    & 102$\pm$15                                        & 44$\pm$12                                   & 47$\pm$10                               \\
                         & \multicolumn{1}{c|}{LM-\RRT*}    & 420$\pm$67                                       & 190$\pm$50                               & 220$\pm$45                             \\
                         & \multicolumn{1}{c|}{Informed-\RRT*} & 110$\pm$13                           & 29$\pm$9.2                               & 73$\pm$8.5                             \\
                         & \multicolumn{1}{c|}{\PRM*}       & \textbf{19$\pm$0.34}                              & N/A                                         & \textbf{9.1$\pm$0.34}                   \\ \bottomrule
\end{tabular}
\end{table}

In this section, we evaluate the performance of our method in simulated environments with various degrees of complexity: a standard room floor-plan
% \footnote{Available at  \makebox[0pt][l]{http://robotang.co.nz/projects/robotics/custom-player-plugins}}
(\cref{fig:map:room}), 
maze
% \footnote{Available at http://www.mazegenerator.net/Examples.aspx}
(\cref{fig:map:maze}), and clutter (\cref{fig:map:clutter});
listed in decreasing order of their $(\alpha,\beta,\epsilon)$-\emph{expansiveness} \autocite{hsu1997_PathPlan}.
Experiments were performed with 20 different pairs of start and goal locations, each repeated 20 times. Planning was performed with a fixed budget of nodes, and planners could run until the budget is exhausted.
The pairs of locations with the highest cost were plotted in \cref{fig:alg-comparasion}.
\RRdT*, 
\RRT* \autocite{karaman2010_IncrSamp},
Bi-\RRT* \autocite{kuffner2000_RRTcEffi}, 
Informed \RRT* \autocite{gammell2018_InfoSamp},
Learning-based Multi-RRTs (LM-RRT) \autocite{wang2018_LearMult},
and \PRM* \autocite{karaman2011_SampAlgo} 
were implemented under the same planning framework\footnote{Code and map resources available at https://github.com/soraxas/RRdT}
in Python.
Note that since \PRM* is not an anytime algorithm \autocite{karaman2011_AnytMoti}, a solution is only available after the entire budget is exhausted.
Therefore, for comparing \RRdT* against an equivalent multi-query planner, the result of \PRM* presented here reconstructs its entire graph every 250 iterations (and discarded afterwards).
Hence, the time-to-solution from \PRM* should be regarded as the slowest, as it is only available after the full budget.

We used the number of sampled points as a metric to measure sampling efficiency.
\Cref{table:sampled-pts-results} shows results of sampled points by the various SBPs.
Because \PRM* does not attempt to create connections during sampling, the \textit{fail to connect} measurements is not applicable.
All planers behave similarly in the simple \emph{Room} environment.
In more complex environments like the \emph{Maze} and \emph{Clutter}, \RRdT* has the least amount of failed connections compared to other incremental SBPs.
In fact, sampling efficiency of \RRT* and Bi-\RRT* in \Cref{table:sampled-pts-results} are identical.
However, as depicted in \cref{fig:alg-comparasion}, Bi-\RRT* outperforms \RRT* in time-to-solution, which is indicative of the benefits of growing multiple trees from different roots.
Indeed, \RRdT* significantly outperforms \RRT* and Bi-\RRT* in finding an initial solution in environments with limited visibility.
LM-\RRT* is also benefited from using multiple trees.
It was able to achieve a fast solution in \emph{Room}, but obtained inferior results in \emph{Maze} and \emph{Clutter} due to no apparent narrow passages.
The planner suffered from choosing a suitable tree to add the sampled points because the placement of its trees is hindered by environments without well-defined narrow passages.
On the other hand, d-trees in \RRdT* exploits local spaces with the ability to spawn new roots when the current one appears to be stuck.
% In the event that such space brings no new connections, \RRdT* can restarts the arm at other locations while 
It retains the previous d-tree for re-connection later, which are valuable information that other incremental SBPs throw away during failed tree expansion.

\Cref{fig:map:maze} highlights the different behaviour between different incremental SBPs.
Both \RRT* and Bi-\RRT* are restricted by surrounding obstacles which limit visibility.
Whereas \RRdT* has a faster time-to-solution (\cref{fig:alg-comparasion}) since it explores spaces incrementally and evenly.
% The MCMC local sampling employed by \RRdT* reduces failed connections and can learn to be guided by obstacles.
% since \PRM*'s visibility between milestones is blocked by the straight walls, rendering inferior result in creating linkages.
Informed \RRT* has an improved convergence time and can reduce the failed connections after finding a solution.
However, it does not improve the time it takes to find an initial solution.
Informed \RRT* has the same time-to-solution as \RRT*, as the dynamic sampling only happens after an initial solution is found.
% Thus, in a normal setting, it would never outperform \RRdT*.
On the other hand, \RRdT* fuses the advantageous properties of \RRT* and \PRM*.
\RRdT* maintains similar sample counts as \PRM*, but being an incremental SBP, it can return an initial solution faster.

\section{Conclusion}\label{sec:conclusion}

We presented \RRdT*, an incremental multi-query SBP that actively balances global exploration and local-connectivity exploitation, formulated as an MAB problem.
By exploiting local connections, \RRdT* maintains sample efficiency
% and high visibility 
which produce robust performance in highly complex spaces.

Limited visibility, especially within narrow passages, has a restricting effect on tree expansions in incremental planners.
\RRdT* mitigates it by exploring \emph{C-space} with multiple exploring disjointed-trees, which on their own exploits local-connectivity.
High visibility of \RRdT* saves computational resources, allowing it to exploits local structures and behave consistently even in a highly constrained environment.
Active exploration combined with MCMC Random Walk exploitation is a novel approach that brings sampling efficiency to incremental SBPs, whilst keeping theoretical guarantees.

We believe that \RRdT* adaptively balances the exploration-exploitation trade off during planning, making it robust to the complexity of \emph{C-space}.
% We believe the exploration-exploitation balances makes \RRdT* highly adaptive and robust towards different complexity of \emph{C-space}.
This model can be improved by a dynamic sampling procedure instead of a uniform distribution.
% In future, a dynamic sampling procedure can be used in place of the generic uniform sampling for further enhancement.
Exploring with multiple trees in \RRdT* can also be easily implemented using parallel programming.
% The concept of an active balance can also be applied in a wide range of SBPs, to utilise resources available more efficiently.

\addtolength{\textheight}{-8.1cm}   % This command serves to balance the column lengths

%%%%%%%%%%%%%%%%%%%%%%%%%%%%%%%%%%%%%%%%%%%%%%%%%%%%%%%%%%%%%%%%%%%%%%%%%%%%%%%%

\clearpage

\printbibliography

\end{document}